\documentclass[onecolumn,10pt]{IEEEtran}
\usepackage{ifpdf}
\usepackage{cite}
\usepackage{graphicx}
\usepackage[colorlinks,linkcolor=black,anchorcolor=black,citecolor=black]{hyperref}
\usepackage{float}
\usepackage{booktabs}
\usepackage{multirow}
\usepackage{algorithmic}
\usepackage[normalem]{ulem}
\useunder{\uline}{\ul}{}
\usepackage[linesnumbered,ruled]{algorithm2e}

\usepackage{amsmath,amsfonts,amssymb,amsopn,amscd,bbm}
\usepackage{lineno}
\newtheorem{definition}{Definition}

\newtheorem{proposition}{Proposition}

\newtheorem{proof}{Proof}
\begin{document}
\title{Label-Informed Outlier Detection Based on Granule Density}

\author{Baiyang Chen, Zhong Yuan, Dezhong Peng, Hongmei Chen, Xiaomin Song, and Huiming Zheng
    \thanks{This work was supported by the National Natural Science Foundation of China (62306196 and 62376230), the Sichuan Science and Technology Program (2024NSFTD0049, 2024YFHZ0144, 2024YFHZ0089 and 2024NSFSC0443
    ), and the Fundamental Research Funds for the Central Universities (YJ202245).}
    \thanks{Baiyang Chen, Zhong Yuan, and Dezhong Peng are with College of Computer Science, Sichuan University, Chengdu 610065, China; 
    Hongmei Chen is with School of Computing and Artificial Intelligence, Southwest Jiaotong University, Chengdu 611756, China; 
    Xiaomin Song, Huiming Zheng and Dezhong Peng are with Sichuan National Innovation New Vision UHD Video Technology Co., Ltd., Chengdu 610095, China.
    }
    \thanks{Corresponding author: Zhong Yuan, E-mail: yuanzhong@scu.edu.cn}
    }
    
    \maketitle
    
\begin{abstract}
    Outlier detection, crucial for identifying unusual patterns with significant implications across numerous applications, has drawn considerable research interest. {Existing semi-supervised methods typically treat data as purely numerical and} in a deterministic manner, thereby neglecting the heterogeneity and uncertainty inherent in complex, real-world datasets. This paper introduces a label-informed outlier detection method for heterogeneous data based on Granular Computing and Fuzzy Sets, namely Granule Density-based Outlier Factor (GDOF). Specifically, GDOF first employs label-informed fuzzy granulation to effectively represent various data types and develops granule density for precise density estimation. Subsequently, granule densities from individual attributes are integrated for outlier scoring by assessing attribute relevance with a limited number of labeled outliers.
    Experimental results on various real-world datasets show that GDOF stands out in detecting outliers in heterogeneous data with a minimal number of labeled outliers.
    The integration of Fuzzy Sets and Granular Computing in GDOF offers a practical framework for outlier detection in complex and diverse data types. 
    All relevant datasets and source codes are publicly available for further research.\footnotetext{
    This is the author’s accepted manuscript of a paper published in \emph{IEEE Transactions on Fuzzy Systems}.
    The final version is available at \url{https://doi.org/10.1109/TFUZZ.2024.3514853}.
    }
    \end{abstract}
                
\begin{IEEEkeywords}
    Outlier detection, Fuzzy sets, Granular computing, Label-informed fuzzy granulation, Granule density, Heterogeneous data
    \end{IEEEkeywords}
    \IEEEpeerreviewmaketitle

\section{Introduction}
    Outlier detection (OD), also referred to as anomaly detection, plays a pivotal role in identifying data points or patterns that significantly {deviate} from the expected or normal behavior. 
    This process is fundamental across various domains, such as fraud detection \cite{pourhabibi2020fraud}, software defect prediction \cite{jiang2022random}, industry control \cite{wang2019outlier}, medical outlier detection \cite{Hawkins1981Identification}, due to its ability to reveal unusual or unexpected behaviors with important implications for the system under study.
 
    Many existing detection methods \cite{Breunig2000LOF, zhang2009anew, Jiang2010An, li2020copod, almardeny2022anovel, li2022robust} are unsupervised due to the scarcity of labeled outliers. While these methods have been effective to an extent, their performance heavily depends on how well their underlying assumptions align with the specific nature of the outliers in the data \cite{han2022ADbench}. This limitation has led to the exploration of semi-supervised detection algorithms \cite{Pang2018RAMODO, Zhao2018XGBOD, Pang2019DevNet, Pang2023PReNet, Ruff2020DeepSAD, Huang2021ESAD, Zhou2022FEAWAD, Chen2024COD}, which utilize a small number of labeled outliers to guide the detection process. 
    However, a notable challenge in these semi-supervised methods is their primary focus on numerical data with deterministic approaches. In practice, datasets often comprise {heterogeneous attributes }\cite{Zhang2016mixeddata, Chen2024COD} where the {attribute values have multiple data types}, such as numerical, categorical, and textual.
    These detectors tend to treat mixed data as purely numerical, thereby imposing artificial properties that may distort the true nature of the data. Furthermore, they usually rely on deterministic approaches to construct detection models \cite{Yuan2023WFRDA}, which may result in overlooking important aspects such as uncertainty, imprecision, or inherent fuzziness present in data. These practices underscore the necessity for an approach adept at managing heterogeneity and uncertainty in data comprehensively.

    Granular Computing (GrC), a framework aimed at simplifying complex problems by breaking them down into smaller, more manageable sub-problems represented by information granules, offers a promising solution. Within this framework, Fuzzy Sets (FS) stand out for their ability to construct information granules with flexible boundaries defined by fuzzy membership functions. 
    {The integration of GrC and FS has shown success in various research areas\cite{Yao2013granular}, including clustering \cite{Pedrycz2010fuzzy}, feature selection \cite{Zhang2016mixeddata, Hu2018large}, outlier detection \cite{yuan2022FRGOD, Chen2024MFIOD}, etc. Our work leverages this approach for two primary benefits:
    Firstly, FS provides the flexibility needed to represent a wide array of data types, from numerical values to categorical variables in its original, eliminating the need for data type transformation.
    Secondly, FS effectively handles} the inherent fuzziness in data by allowing elements to have degrees of membership within a set, offering a more effective way to model data that exhibits uncertainty or ambiguity. 
        
    In this context, we introduce the Granule Density-based Outlier Factor (GDOF), a novel label-informed method for outlier detection in heterogeneous data.
    GDOF leverages the principles of GrC and FS to construct fuzzy granules {and utilizes a small number of labeled outliers to enhance detection performance.} At the core of our approach is the formulation of granule density, which incorporates both local neighborhood information and global density patterns for precise density estimations. This approach allows us to effectively discern between normal and anomalous data patterns with the assumption that outliers typically reside in low-density areas. Subsequently, GDOF integrates granule densities from each attribute by assessing its attribute relevance, which reflects the discriminative power of each attribute in identifying outliers. {Finally, an outlier factor is formulated based on granule densities for each object, indicating the likelihood of the object being an outlier.}
    The main contributions of this paper are:
    \begin{itemize}
        \item This paper proposes a novel label-informed fuzzy granulation method that enables a downstream task-oriented approach to construct fuzzy granules for outlier detection, potentially advancing the application of GrC in real-world contexts.
        \item We introduce the granule density that integrates local neighborhood information and global density patterns for effectively characterizing outliers in heterogeneous data.
        \item We develop a label-informed outlier detection algorithm based on granule density. The algorithm integrates granule densities from each attribute by assessing the attribute relevance with a limited number of labeled data.
        \item Comprehensive experimental validation on various types of public datasets demonstrates that our proposed algorithm performs on par with or surpasses existing state-of-the-art methods.
        \end{itemize}

\section{Related works}
\subsection{Density-based approaches}
    Density-based detection algorithms have drawn significant research interest for their adept handling of data with varying densities and their resilience to diverse data distributions. Generally, these methods operate under the assumption that outliers tend to be situated in regions of low data density. Breunig et al. \cite{Breunig2000LOF} first introduced such a method, the Local Outlier Factor (LOF), which identifies outliers by assessing the local density deviation of a data point relative to its neighbors, particularly in areas of varying densities.
    Following LOF, a range of density-based approaches have been extensively explored.
    For example, the Connectivity-based anomaly Factor (COF) \cite{Tang2002COF} improves LOF by considering the neighbors' connectivity in addition to their density.
    The Local Correlation Integral (LOCI) \cite{Papa2003LOCI} quantifies the anomaly clusters by assessing the multi-granularity deviation in the density within a specified radius of a point. 
    Subsequently, Ren et al. \cite{Ren2004RDF} developed the Relative Density Factor (RDF), which detects outliers by excluding data points deeply embedded within clusters, thereby refining the selection of potential outliers.
    Jin et al. \cite{Jin2006INFLO} proposed the Influenced Outlierness (INFLO) that estimates data density by incorporating both nearest neighbors and reverse nearest neighbors of objects, broadening the context for outlier analysis.
    Later, Tang and He \cite{TANG2017RDOS} improved INFLO on the way of density estimation by incorporating shared nearest neighbors and presented the Relative Density-based Outlier Score (RDOS) based on local kernel density estimation.
    Li et al. \cite{li2022robust} proposed the Directed density-ratio Changing Rate-based Outlier Detection (DCROD) that combines kernel density estimation with an extended neighborhood to rank objects.
    Incorporating probabilistic principles into density-based methods has also yielded notable improvements. For instance, Kriegel et al. \cite{Kriege2009LoOP} combined density-based local outlier scoring with a probabilistic approach in their Local Outlier Probability (LoOP) method.
    Goldstein et al. \cite{Goldstein2012HBOS} developed the Histogram-based Outlier Score (HBOS), employing histograms for each feature and determining outlier scores based on the probability density of the corresponding bins.
    More recently, Yuan et al. \cite{Yuan2023WFRDA} leveraged Fuzzy Rough Sets for granule density estimation and proposed the Weighted Fuzzy-Rough Density-based Anomaly detection algorithm (WFRDA).
    While these methods have made significant progress, their performance largely depends on the alignment of their underlying assumptions with the specific nature of the outliers in the dataset.

\subsection{Semi-supervised approaches}
    The limitation of the unsupervised approaches has led to the exploration of semi-supervised outlier detection methods \cite{Jiang2023WSAD}, which utilize a few labeled outliers to guide the detection process.
    These semi-supervised approaches can be roughly categorized into the following branches: rank-based, representation learning-based, active learning-based, and reinforcement learning-based methods. 
    
    Rank-based methods typically train a ranking model that rates the extent to which an object is an outlier.
    For example, Pang et al. \cite{Pang2019DevNet} introduced an end-to-end deep learning framework known as DevNet, which incorporates a Gaussian prior and the deviation loss for predicting outlier scores.
    Following DevNet, Zhou et al. \cite{Zhou2022FEAWAD} developed FEAWAD, a weakly supervised outlier detection method that uses a deep autoencoder to model normal data patterns. The features learned by this autoencoder are then used to enhance the accuracy of outlier score prediction.
    Additionally, the PReNet model, introduced by Pang et al. \cite{Pang2023PReNet}, employs ordinal regression to rate objects without making any assumptions about the distribution of outliers.  
    
    Representation learning-based models offer an indirect approach to determining outlier scores, typically enhancing unsupervised learning with the aid of labeled outliers.    
    Early solutions, such as OE (Outlier Ensemble) \cite{2014OE} and XGBOD (eXtreme Gradient Boosting for Outlier Detection) \cite{Zhao2018XGBOD}, leverage multiple unsupervised detectors to extract meaningful features and then combine the resulting outlier scores with original features for training a supervised classifier.
    More recently, deep learning techniques have emerged in semi-supervised outlier detection. They generally employ end-to-end frameworks to extract outlier-focused representations. One notable example is DeepSAD, introduced by Ruff et al. \cite{Ruff2020DeepSAD}, which extends the unsupervised detector DeepSVDD (Deep Support Vector Data Description) \cite{Ruff2018DeepSVDD}. DeepSAD incorporates labeled data and penalizes the inverse distances of outlier representations, encouraging them to be projected further away from normal data.
    Following DeepSAD, Huang et al. \cite{Huang2021ESAD} presented an encoder-decoder-encoder architecture (ESAD) {that incorporates information entropy} to detect outliers.
    The other detector REPEN \cite{Pang2018RAMODO} introduces a ranking approach for representation learning of ultra-high dimension data. A few labeled outliers are leveraged to refine the sampling for the training triplets and help the model learn outlier-oriented representations. 
    {Another research direction explores the utilization of generative adversarial networks (GANs) to enhance model training and/or data augmentation. Notably,} Tian et al. \cite{Tian2022AA-BiGAN} propose the anomaly-aware bidirectional GAN model (AA-BiGAN), {which learns to assign low-density values to anomalies.}
    Li et al. \cite{Li2022Dual-MGAN} integrate multiple GANs to realize reference distribution construction and data augmentation for detecting both discrete and grouped outliers.    

    Active learning (AL)-based methods {can acquire more labeled data with human assistance. For instance, Gornitz et al.\cite{Gornitz2013SSAD} propose a detection method that chooses borderline data to label. Another AL-based method is Active Anomaly Discovery (AAD) \cite{Das2016AAD}, which greedily selects the most likely outliers to label} and maximizes the number of true outliers under a query budget. 
    {While the AL mechanism offers flexibility in model development,} the process relies on human labeling, which can become a bottleneck if the data volume is very large or if the labeling task is complex.
    
    Reinforcement learning (RL)-based methods typically consider outlier detection as a sequential decision-making process.
    For instance, the model Meta-AAD, proposed by Zha et al. \cite{Zha2020Meta-AAD}, employs deep reinforcement learning to optimize a meta-policy for selecting the most appropriate samples for manual labeling. {In contrast, Pang et al. \cite{Pang2021DPLAN} introduce} a deep Q-learning-based model DPLAN that utilizes unlabeled data without human help. DPLAN creates an outlier-biased simulation environment to enable outlier detection {in an RL context}. 
    RL-based methods excel in situations where outlier detection can be framed as a sequence of decisions or actions, capturing temporal or sequential relationships in the data. However, they are generally more complex and may require significant computational resources for training and optimization, especially for large-scale problems. Moreover, designing an effective reward system is critical yet challenging, as it must accurately represent the detection goals. 

\section{Preliminaries}
    This section reviews {some basic concepts of Fuzzy Sets to help understand the main content} of this paper.
    
\begin{definition}
    An information system is a tuple $(X, A)$, where $X=\{x_1,x_2,\ldots,x_n\}$ is the set of objects, and $A=\{a_1,a_2,\ldots, a_m\}$ is the set of attributes {associated with each object.}
    \end{definition}

\begin{definition}
    Given an information system $(X, A)$, if $\widetilde{Y}$ is a {mapping} from $X$ to $[0,1]$, then $\widetilde{Y}$ is a fuzzy set on $X$, i.e. $\widetilde{Y}:X\to [0,1]$.
    \end{definition}

    $\forall x_i \in X$, $\widetilde{Y}(x_i)$ is the membership of $x_i$ to $\widetilde{Y}$, or the membership function of $\widetilde{Y}$. The fuzzy set is often denoted by $\widetilde{Y}=\left(\widetilde{Y}({{x}_{1}}),\widetilde{Y}({{x}_{2}}),\ldots ,\widetilde{Y}({{x}_{n}})\right)$, and the cardinality of  $\widetilde{Y}$ is computed by $|\widetilde{Y}|=\sum\limits_{i}{{\widetilde{Y}}({x}_{i})}$.
 
\begin{definition}
    Let $X$ be a set of objects, a fuzzy relation $\widetilde{R}$ on $X$ is a family of fuzzy sets $\widetilde{R}: X\times X \rightarrow [0,1]$.
    \end{definition}
    
    $\forall (x_i,x_j)\in X\times X$, the membership {function $\widetilde{R}(x_i,x_j)$ indicates} the degree to which $x_i$ has a relation $\widetilde{R}$ with $x_j$. A fuzzy relation $\widetilde{R}$ is usually denoted by a fuzzy relation matrix $M({\widetilde{R}})=(r_{ ij})_{n\times n}$, where $r_{ij}=\widetilde{R}(x_i,x_j)$.
    $\forall x_i, x_j \in X$, if a fuzzy relation $\widetilde{R}$ {satisfies: (1) reflexivity: $\widetilde{R}(x_i, x_i)=1$, and (2) symmetry: $\widetilde{R}(x_i, x_j)=\widetilde{R}(x_j, x_i)$, then $\widetilde{R}$ is also a fuzzy similarity relation.} 
    
\begin{definition}\cite{Qian2015fuzzy}
    The fuzzy granular structure of $X$ induced by a fuzzy similarity relation $\widetilde{R}$ is
    \begin{equation}
    G_{\widetilde{R}}(X)=\left\{[x_i]_{\widetilde{R}}\right\}_{x_i\in X},
    \end{equation}
    where $[x_i]_{\widetilde{R}}=\left(r_{i1}^{\widetilde{R}}, r_{i2}^{\widetilde{R}}, \dots, r_{in}^{\widetilde{R}}\right)$ is a fuzzy granule containing the object $x_i$. 
    \end{definition}
    
    One can observe that each fuzzy granule $[x_i]_{\widetilde{R}}$ is a fuzzy set, which {reflects} how similar the object $x_i$ is to all objects in $X$.
    Obviously, $[x_i]_{\widetilde{R}}(x_j)=\widetilde{R}(x_i,x_j)=r^{\widetilde{R}}_{ij}$. If $\widetilde{R}(x_i,x_j)=1$, then it suggests that $x_j$ certainly belongs to $[x_i]_{\widetilde{R}}$; If $\widetilde{R}(x_i,x_j)=0$, then $x_j$ definitely does not belong to $[x_i]_{\widetilde{R}}$. The fuzzy cardinality of $[x_i]_{\widetilde{R}}$ is calculated by
    \begin{equation}\label{eq_cardinality}
        \big|[x_i]_{\widetilde{R}}\big|=\sum\limits_{j=1}^{n}{{\widetilde{R}}({{x}_{i}},{{x}_{j}})}.
        \end{equation}
    We can easily obtain $1\le \big|[x_i]_{\widetilde{R}}\big|\le |X|$. {The cardinality of $[x_i]_{\widetilde{R}}$ reflects the overall similarity of the object $x_i$ to others in $X$ based on the knowledge $\widetilde{R}$.}
    
\section{Methodology}
    This section details the proposed detection method GDOF. Initially, we utilize a subset of labeled data to construct fuzzy information granules for a versatile representation of heterogeneous data. Subsequently, we delve into the concept of granule density and introduce the concept of attribute relevance for assessing the discriminative power of specific attributes with labeled outliers. Lastly, we define the outlier factor based on granule density and predict outliers in data with a threshold. %The overall framework of GDOF is illustrated in Figure \ref{fig_GDOF}.
    
    % \begin{figure*}[!ht]
    %     \centering
    %     \includegraphics[width=\textwidth]{figs_model_GDOF.pdf}
    %     \caption{Overall framework of GDOF.}
    %     \label{fig_GDOF}
    %     \end{figure*}
        
\subsection{Label-informed fuzzy granulation}
The primary concern of outlier detection within heterogeneous data is the effective representation of diverse data types. We adopt fuzzy information granulation, a concept derived from GrC and FS, to tackle this issue. This approach involves breaking down a complex dataset into simpler, more manageable fuzzy granules, where the boundaries of granules are not rigid but fuzzy. Fuzzy granulation allows data points to be associated with multiple granules to varying degrees, aptly capturing the inherent vagueness and ambiguity often present in real-world data.
However, the construction of these granules, especially determining their optimal size (or radius), has been somewhat heuristic (e.g., coverage and specificity \cite{Pedrycz2016Design}), which might not adapt well to diverse or changing data characteristics. 
We address this limitation by incorporating labeled data into the granulation process.
This integration ensures that the construction of granules is not only reflective of the actual data distribution but also relevant and meaningful for specific tasks, such as outlier detection.
%Label-informed Fuzzy Granulation (LFG) primarily focuses on optimizing the radius parameter of fuzzy granules, a crucial factor that determines the granularity and, consequently, the effectiveness of fuzzy models in various applications. The core rationale behind LFG is that labeled data provides concrete, empirical insights into the system being modeled, allowing for a more informed and precise determination of the fuzzy granule radius. By integrating this data, LFG can adaptively adjust the granularity of the model to better capture the underlying patterns and relationships within the data, enhancing both the model's interpretability and its predictive accuracy.
%This adaptability ensures that the granules are appropriately sized to capture the essential information in the data, neither too coarse to lose important details nor too fine to be over-sensitive to noise.

This part starts with the formulation of fuzzy relations, based on which the fuzzy granules are subsequently constructed under the guidance of labeled outliers. Additionally, we discuss a negative sampling method for scenarios where labeled normal data are not available.

\subsubsection{Fuzzy relations}
% Fuzzy relations are often used to define the degree of association or similarity between elements, which is a fundamental step in forming fuzzy granules. Essentially, fuzzy relations can provide the basis on which fuzzy granules are constructed.
% In granular computing, which involves breaking down complex systems into simpler parts (granules), fuzzy relations are crucial in defining the criteria for how these granules are formed and related to each other. This process often leverages the principles of fuzzy logic to handle the imprecision in data grouping and relationships.
    Fuzzy relations play a pivotal role in fuzzy systems by quantifying the degree of association or similarity between elements. They are crucial in defining the criteria for how information granules are formed. The membership function for fuzzy relations can be defined in various ways. We utilize a hybrid fuzzy membership function \cite{yuan2022FRGOD} to jointly represent multiple data types. Specifically, We mainly recognize three broad data types, i.e., the numerical data, the categorical data, and the combination of the two types (i.e., the mixed type). Due to the varying magnitudes and dimensions of raw numerical data, we first normalize these values into the range of $[0, 1]$ using min-max normalization.

    Let $f_i^k$ be the value of attribute $a_k$ for object $x_i$,
    and $d_{ij}^k=\left|f_i^k - f_j^k\right|$ is the difference between $x_i$ and $x_j$ with respect to (w.r.t.) attribute $a_k$, the fuzzy relation $\widetilde{a}_k(x_i,x_j)$ between $x_i$ and $x_j$ induced by a single attribute $a_k$ is calculated by
    \begin{equation}\label{eq_rel_ak}
        \widetilde{a}_k(x_i,x_j)=\left\{
        \begin{array}{ll}
        \mathbb{I}_{f_i^k=f_j^k},   & \text{if } a_k \text{ is categorical};\\
        1-d_{ij}^k, & \text{if } d_{ij}^k\le \lambda^k, a_k \text{ is numerical};\\
        0,          & \text{if } d_{ij}^k > \lambda^k, a_k \text{ is numerical};\\
        \end{array}\right.
        \end{equation}
    where $\mathbb{I}_{(\cdot)}$ represents an indicator function that returns 1 if the condition is met; otherwise, it returns 0. $\lambda^k\in [0,1]$ is a radius parameter. As the fuzzy relation $\widetilde{a}_k$ satisfies reflexivity: $\widetilde{a}_k(x_i, x_i)=1$ and symmetry: $\widetilde{a}_k(x_i, x_j) = \widetilde{a}_k(x_j, x_i)$, we have that $\widetilde{a}_k$ is a fuzzy similarity relation.
    
    Subsequently, the fuzzy relation $\widetilde{B}$ of an attribute set $B\subseteq A$ can be formulated through fuzzy operations, among which the conjunction operation is most widely adopted \cite{Yuan2021ARsurvey}.

    \begin{equation}\label{eq_rel_B}
        \widetilde{B}(x_i, x_j)=\min_{a_k\in B}{{\widetilde{a}_k}(x_i, x_j)}. 
    	\end{equation}
    \begin{proposition}\label{coro_subset}
        Given an information system $(X, A)$, for any attribute subset $B,C\subseteq A$, if $C\subseteq B$, then $\widetilde{B} \subseteq \widetilde{C}$.
    \end{proposition}
    \begin{proof}
        Given that $C\subseteq B$, according to the conjunction rule (Eq.~\ref{eq_rel_B}), $\forall{x_i, x_j \in X}$, we have  $\widetilde{B}(x_i,x_j)\leq \widetilde{C}(x_i,x_j)$. Hence, we obtain $\widetilde{B} \subseteq \widetilde{C}$.
    \end{proof}
    
    Proposition \ref{coro_subset} {describes} the inclusion relation between fuzzy relations induced by attributes; {specifically,} the more attributes adopted, the smaller (fine-grained) the fuzzy relation. 
    %Notably, given any object $x_i, x_j \in X$, if the fuzzy relation $\widetilde{B}$ is sufficiently fine-grained, the degree to which $x_i$ has a relation $\widetilde{B}$ with $x_j$ approaches 0. In other words, every object in $X$ will be distinct from each other under the knowledge $\widetilde{B}$.
    
\subsubsection{Constructing fuzzy granules}
    % In Fuzzy Sets, a fuzzy granule can be induced by a fuzzy relation. 
    The fuzzy granule centered at an object $x_i$ w.r.t. the fuzzy relation $\widetilde{B}$ can be formulated as
    \begin{equation}
    [x_i]_{\widetilde{B}} = \left(r^{\widetilde{B}}_{i1}, r^{\widetilde{B}}_{i2}, \ldots, r^{\widetilde{B}}_{in}\right). 
        \end{equation}
    In our label-informed framework, we introduce a task-guided approach to determine the granule radius using the labeled data. Specifically, the optimal radius $\lambda^k$ is derived through the maximization of the following objective:
    \begin{equation}\label{eq_lamb}
        \arg\max_{\lambda^k} \frac{\sum\limits_{x_i\in X_-}\sum\limits_{y\in X}\widetilde{a}_k(x_i,y)}{|X_-|\cdot |X|} - \frac{\sum\limits_{x_j\in X_+}\sum\limits_{y\in X}\widetilde{a}_k(x_j,y)}{|X_+|\cdot |X|},
        \end{equation}
    where $X_-$ and $X_+$ denote the subset of the negative instances (i.e., the normal objects or inliers) and positive instances (outliers) in $X$, respectively. 
    The underlying assumption is that outliers are less similar to others than the majority of data. Therefore, they tend to have lower fuzzy relation with others on average based on the knowledge ${\widetilde{a}_k}$. This objective is to find a balance between the labeled outliers and inliers. The optimal $\lambda^k$ for each attribute $a_k$ can be obtained by Dynamic programming or {Genetic Algorithms}.
    
\subsubsection{Negative sampling}
    In our label-informed setting, a few outliers are known. However, the normal objects $X_-$ may not be explicitly presented. In this case, it is necessary to designate certain instances as candidate negative instances. Most semi-supervised detectors \cite{Pang2023PReNet, Zhou2022FEAWAD, Pang2018RAMODO} simply treat unlabeled data as inliers {based on the fact that outliers are very few. However, this approach may be less effective in noisy situations with a substantial number of outliers.}
    
    This paper introduces a negative sampling method based on the average distance of objects. Formally, let $D_i^k=\frac{1}{|X|}{\sum_{x_j\in X}d_{ij}^k}$ be the average distance of an object $x_i$ to all objects in $X$ on attribute $a_k$. We randomly sample $N_-$ instances from the unlabeled objects $X-X_+$ with a probability:
    \begin{equation}\label{eq_neg_sampl}
        p(x_i)=\text{softmax}\left(1-\frac{1}{|A|}{\sum_{a_k\in A}D_i^k}\right),
    \end{equation}
    where the softmax function is a non-linear transformation that ensures that the probabilities $p(x_i)$ are between 0 and 1 and sum up to 1, making it suitable for representing a probability distribution. 
    
    This approach results in the inlier set $X_-$.  The underlying rationale for this sampling method is the common assumption that inliers tend to be closer (i.e., more similar) to the majority of the data, thereby having a lower average distance $D_i^k$. Consequently, objects with smaller average distances are more likely to be inliers and are thus sampled with a higher probability.

\subsection{Granule density}
    Density-based outlier detection methods are popular for their adaptability to varying data densities and robustness against diverse distributions. These methods typically estimate data densities and identify outliers in low-density regions. 
    Since fuzzy granulation provides a flexible representation that adapts to diverse data types with inherent vagueness and ambiguity, we introduce the concept of fuzzy granule density (GD) as a means to characterize outliers. 
    
    % Fuzzy Sets theory allows for processing heterogeneous data directly, maintaining valuable information for knowledge discovery. This is particularly advantageous in outlier detection for heterogeneous data, where data types vary in scale, distribution, and interpretation. Therefore, we employ fuzzy granules to represent such data and define granule density (GD) as a means to characterize outliers. 
    % Drawing on the concept of fuzzy-rough density introduced by Yuan et al.\cite{Yuan2023WFRDA}, we extend it with a local density measure for a comprehensive view of an object's data structure. % We distinguish between global fuzzy granule density (GGD) and local density of fuzzy granules to provide a more comprehensive view of the data structure.
    Yuan et al. introduced the concept of fuzzy density in their work \cite{Yuan2023WFRDA}, where they propose the notion of fuzzy-rough density for an object $x_i$ w.r.t. to an attribute $a$. This is defined as the fuzzy cardinality of a fuzzy granule, given by $Den_a(x_i)=\frac{1}{|X|}|[x_i]_a|$. This definition incorporates a global perspective by evaluating the granule's cardinality relative to the size of the entire dataset, yet it does not account for the density contributions from the object's local neighbors.
    To capture the local structure, we introduce the concept of the local granule density (LGD).   
       
    \begin{definition}\label{def_LGD}
        Given an attribute set $B\subseteq A$, the local granule density (LGD) of the object $x_i$ w.r.t. the attribute set $B$ is defined as
        \begin{equation}\label{eq_LGD}   
            LGD_B(x_i)={\frac{{|{[x_i]_{\widetilde{B}}}|}}{\frac{1}{|N_i|}\sum_{x_j\in N_i}|[x_j]_{\widetilde{B}}|}},
        \end{equation}
        where $N_i=\left\{x_j|[x_i]_{\widetilde{B}}(x_j)>0\right\}$ is the set of neighbors of $x_i$.
        \end{definition}

    The neighbors of an object are usually determined by a distance metric. In the fuzzy settings, it is natural to take the instances within the corresponding fuzzy granule. Therefore, the local granule density of $x_i$ is determined by the proportion of a fuzzy granule’s cardinality to the average cardinalities of its neighbors. It is local in that the degree depends on how isolated the object is concerning the surrounding neighborhood. Combining the global and local granule density, we have the definition of granule density in the following.

    \begin{definition}\label{def_GD}
        Let $(X,A)$ be an information system and $B\subseteq A$, the granule density (GD) of $x_i \in X$ w.r.t. the attribute set $B$ is defined as
        \begin{equation}\label{eq_GD}   
            GD_B(x_i)= \frac{1}{|X|}|{[x_i]_{\widetilde{B}}}| \cdot LGD_B(x_i).
        \end{equation}
        \end{definition}
        
    In the above definition, %the granule density of $x_i$ on the attributes $B$ is determined by the product of the global and local density, 
    $\frac{1}{|X|}|{[x_i]_{\widetilde{B}}}|$ represents the global density information of $x_i$ on the attribute set $B$, and the integrated granule density of an object is determined by the product of the global and local density, 
    thereby providing a more comprehensive description of the data structure. Next, we provide a detailed analysis of the properties of the granule density.

\begin{proposition}\label{prop_gd}
    Let $D$ be a set of objects with {an attribute $a$, if $\forall x_i, x_j \in D$, $d_{ij}^a<\delta$ (i.e., $D$ is a considerably dense area), then} it holds that $(1-\delta)^2<GD_a(x_i)<\frac{1}{1-\delta}$.
    \end{proposition}
   
\begin{proof}
    Since $\forall x_i, x_j \in D$, $d_{ij}^a<\delta$, we have 
    $\widetilde{a}(x_i, x_j)>1-\delta$. Therefore, 
    \[|D|(1-\delta)<|{[x_i]_{\widetilde{a}}}|{\leq}|D|,\] 
    \[|D|(1-\delta)<\frac{1}{|D|}\sum_{x_j\in D}|[x_j]_{\widetilde{a}}|<|D|.\]    
    Then we have
    \[1-\delta<\frac{1}{|D|}|{[x_i]_{\widetilde{a}}}|<1,\]
    \[1-\delta<\frac{|{[x_i]_{\widetilde{a}}}|}{\frac{1}{|D|}\sum_{x_j\in D}|[x_j]_{\widetilde{a}}|}<\frac{1}{1-\delta}.\]
    By Definition \ref{def_GD}, it follows that
    \[(1-\delta)^2<GD_a(x_i)<\frac{1}{1-\delta}.\]
    \end{proof}
    
Intuitively, $D$ corresponds to a dense cluster. Proposition \ref{prop_gd} reflects that the granule density of $x_i$ is bounded in dense areas. If $D$ is a tight cluster, the granule radius $\delta$ for the attribute $a$ can be quite small, thus forcing the granule density of $x_i$ to be quite close to 1. 

\begin{proposition}\label{prop_gd_com}
    Given a dense area $D_1$ with attribute $a$, satisfying $\forall x_i, x_j \in D_1$, $d_{ij} < \delta$, {and the area $D_2=D_1\cup \{x_k\}$, if there exists $x_i\in D_1$ such that $d_{ik}>\delta$ (i.e., $D_2$ is less dense than $D_1$), then} it holds that $GD_a^{(D_1)}(x_i )> GD_a^{(D_2)}(x_i)$.
    \end{proposition}
   
\begin{proof}
    Since $\exists x_i\in D_1, d_{ik}>\delta$, we have $\widetilde{a}(x_i, x_k)=0$.    Therefore, 
    \[\left|[x_i]_{\widetilde{a}}^{(D_1)}\right|=\left|[x_i]_{\widetilde{a}}^{(D_2)}\right|, \quad N_i^{(D_1)}=N_i^{(D_2)}=D_1.\] 
    Then we have \[\frac{1}{|D_1|}\left|[x_i]_{\widetilde{a}}^{(D_1)}\right|>\frac{1}{|D_2|}\left|[x_i]_{\widetilde{a}}^{(D_2)}\right|.\] 
    One the other hand, $\forall x_j \in N_i^{(D_1)}$, we obtain
    \[\left|[x_j]_{\widetilde{a}}^{(D_1)}\right|{\leq}\left|[x_j]_{\widetilde{a}}^{(D_2)}\right|.\]
    Therefore, \[\frac{1}{|N_i^{(D_1)}|}{\sum_{x_j\in N_i^{(D_1)}}\left|[x_j]_{\widetilde{a}}^{(D_1)}\right|}{\leq}\frac{1}{|N_i^{(D_2)}|}\sum_{x_j\in N_i^{(D_2)}}\left|[x_j]_{\widetilde{a}}^{(D_2)}\right|.\]
    Hence, we establish $GD_a^{(D_1)}(x_i)>GD_a^{(D_2)}(x_i)$.
    \end{proof}
    
Proposition \ref{prop_gd_com} is intuitive as adding a more distant object $x_k$ to $D_1$ would create a less dense area $D_2$. This proposition demonstrates a consistent trend for the granule density in response to changes in data density. The density decreases when the data becomes less dense or more varied by adding a distant object. This consistency reinforces the validity of the proposed granule density in capturing data density to detect outliers.

\subsection{Granule density-based outlier factor}
    This subsection delves into the utilization of granule density for outlier characterization and the assignment of the anomalous degree of an object, termed outlier factors. 
    The degree to which an object is considered an outlier is signified by its granule densities across multiple feature spaces. Different attributes contribute variably to outlier detection since some attributes might be more indicative of anomalous behavior than others. Therefore, the careful selection and evaluation of an appropriate attribute set are crucial for devising an effective outlier factor. 

    To assess attribute importance, while previous methods might rely on metrics like entropy \cite{Yuan2023WFRDA}, our approach, informed by labeled data, allows for a more {adaptive solution}. We introduce the concept of attribute relevance to evaluate the discriminative power of a specific attribute set $B$ as follows.
    \begin{definition}
    Let $X_-$ and $X_+$ represent the set of labeled inliers and outliers, respectively. For an attribute set $B$, its attribute relevance in outlier detection is defined as 
    \begin{equation}\label{eq_C_B}
        \gamma_{B}= \frac{1}{|X_-|}\sum_{x_i\in X_-}GD_{B}(x_i) - \frac{1}{|X_+|}\sum_{x_j\in X_+}GD_{B}(x_j).
        \end{equation}
        % where $\alpha$ is a balancing parameter that adjusts the relative weight of inliers and outliers.
        \end{definition}
        
    In this definition, the attribute relevance $\gamma_B$ is derived from the granule densities of the labeled data with the assumption that outliers reside in low-density regions. A high value of attribute relevance for a known inlier or a low value for a known outlier, when assessed using subset $B$, signifies the effectiveness of this attribute set in outlier detection. Conversely, if the attribute relevance does not substantially differ between outliers and inliers, it suggests that $B$ may not capture essential distinguishing features for effective outlier detection. 

    To determine the attribute set for designing an outlier factor, a widely accepted approach \cite{Goldstein2012HBOS, li2022ecod, Yuan2023WFRDA} is to assume the independence of each attribute. Following this assumption, we calculate the granule density for each attribute and then aggregate them through their corresponding attribute relevance. This aggregation is executed via a weighted summation, thereby combining the distinct influences of each attribute into a comprehensive outlier factor.

\begin{definition}\label{def_GDOF}
    Given an information system $(X, A)$, the granule density-based outlier factor of the object $x_i\in X$ is defined as
    \begin{equation}\label{eq_GDOF}
        GDOF(x_i)=1- \frac{1}{|A|}\sum_{a_k\in A}\gamma_{a_k}\cdot GD_{a_k}(x_i).
        \end{equation}
        \end{definition}
        
    This definition encapsulates the collective influence of all attributes $A$ in determining the outlier degree of an object $x_i$, providing a holistic measure of outlier.

\subsection{Detection algorithm}
    {Following previous works \cite{Yuan2023WFRDA, Chen2024COD}, our method outputs outliers through thresholding.}
    
    \begin{definition}\label{def_thresh}
        Let $\theta$ be a real-valued threshold, $\forall x_i \in X$, if the outlier factor \textit{GDOF}$(x_i) > \theta$, then $x_i$ is regarded as an outlier.
        \end{definition}

    In the label-informed settings, the threshold value is obtained adaptively. A practical approach is to utilize the average of the smallest GDOF of the labeled outliers and the greatest GDOF of inliers (usually by negative sampling). Let $X_+$ and $X_-$ be the set of labeled outliers and inliers, respectively. Then, the optimal threshold ${\theta}^*$ can be computed as follow:
    \begin{equation}\label{eq_theta}
        {\theta}^*=\frac{1}{2}\left(\min_{x_i\in X_+} {GDOF}(x_i)+\max_{x_j\in X_-} {GDOF}(x_j)\right).
    \end{equation}

\begin{algorithm}[!ht]
    \caption{GDOF}\label{alg_GDOF}
    \LinesNumbered
    \KwIn{An information system $(X, A)$ with a few labeled outliers $X_+\subseteq X$, the negative sampling parameter $N_-$.}
    \KwOut{Outlier set $OS$}
    $OS \leftarrow \emptyset$\;
    Sample negative objects $X_-$ by Eq.(\ref{eq_neg_sampl})\;
    \For{Each $a_k\in A$}{
        Compute granule radius $\lambda^k$ by Eq.(\ref{eq_lamb})\;
        Construct fuzzy granules using Eq.(\ref{eq_rel_ak})\;
        }
    \For{Each $x_i\in X$}{
        \For{Each $a_k\in A$}{
        Compute $GD_{a_k}(X_i)$ by Eq.(\ref{eq_GD})\;
        }
        }
    Compute attribute relevances by Eq.(\ref{eq_C_B})\;
    \For{Each $x_i\in X$}{
        Compute $GDOF(x_i)$ using Eq.(\ref{eq_GDOF})\;
        }                
    Compute the threshold $\theta$ by Eq.(\ref{eq_theta})\;
    \For{Each $x_i\in X$}{
        \If {$GDOF(x_i)>\theta$}{
        $OS \leftarrow OS \cup \{x_i\}$\;
        }
    }
    \Return $OS$
    \end{algorithm}

    Algorithm \ref{alg_GDOF} {outlines the main process of the proposed method GDOF.} Since GDOF requires calculating the granule density of each object for each attribute, the worst time complexity for GDOF is $O(mn)$, where $m$ denotes the number of attributes and $n$ represents the number of objects.

\section{Experiments}
\setcounter{footnote}{0}
In this section, we conduct experiments to evaluate detection algorithms. 
All our datasets and codes are publicly accessible online\footnote{https://github.com/ChenBaiyang/GDOF}.

\begin{table}[!ht]
    \caption{Description of the experimental datasets}\label{tab_data}
    \centering
    \tabcolsep=2pt
    % \resizebox{0.9\textwidth}{!}{
    \begin{tabular}{cccccccc}
    \toprule
    \multirow{2}{*}{No.} &
      \multirow{2}{*}{Datasets} &
      \multirow{2}{*}{\#Objects} &
      \multicolumn{3}{c}{\#Attributes} &
      \multirow{2}{*}{\#Outliers} &
      \multirow{2}{*}{Data type} \\ \cline{4-6}
       &      &       & Num. & Cat. & Total &     &             \\
       \midrule
        1  & Audiology   & 226   & 0         & 69          & 69    & 57  & Categorical \\
        2  & Breast      & 286   & 0         & 9           & 9     & 85  & Categorical \\
        3  & Mushroom1   & 4429  & 0         & 22          & 22    & 221 & Categorical \\
        4  & Mushroom2   & 4781  & 0         & 22          & 22    & 573 & Categorical \\
        5  & Annealing   & 798   & 10        & 28          & 38    & 42  & Mixed       \\
        6  & Arrhythmia  & 452   & 198       & 81          & 279   & 66  & Mixed       \\
        7  & CreditA     & 425   & 6         & 9           & 15    & 42  & Mixed       \\
        8  & Sick        & 3613  & 6         & 23          & 29    & 72  & Mixed       \\
        9  & Thyroid     & 9172  & 7         & 21          & 28    & 74  & Mixed       \\
        10 & Annthyroid  & 7200  & 6         & 0           & 6     & 534 & Numerical   \\
        11 & Breastw     & 683   & 9         & 0           & 9     & 239 & Numerical   \\
        12 & Cardio      & 1831  & 21        & 0           & 21    & 176 & Numerical   \\
        13 & Ionosphere  & 351   & 32        & 0           & 32    & 126 & Numerical   \\
        14 & Mammography & 11183 & 6         & 0           & 6     & 260 & Numerical   \\
        15 & Musk        & 3062  & 166       & 0           & 166   & 97  & Numerical   \\
        16 & Optdigits   & 5216  & 64        & 0           & 64    & 150 & Numerical   \\
        17 & PageBlocks  & 5393  & 10        & 0           & 10    & 510 & Numerical   \\
        18 & Waveform    & 3443  & 21        & 0           & 21    & 100 & Numerical   \\
        19 & Wilt        & 4819  & 5         & 0           & 5     & 257 & Numerical   \\
        20 & Yeast       & 1484  & 8         & 0           & 8     & 507 & Numerical   \\
        \bottomrule
    \end{tabular}
    % }
    \flushleft
    *Note: ``Num.'' and ``Cat.'' refer to the number of numerical attributes and categorical attributes, respectively.
    \end{table}

\subsection{Experiment settings}
The experimental datasets have heterogeneous attributes with various data types, including numerical data, categorical data, and a combination of the two types. {Detailed descriptions of each dataset are provided in Table \ref{tab_data}.}
Following previous studies \cite{zhao2019PyOD, han2022ADbench, Jiang2023WSAD}, we configure the unsupervised detection algorithms to predict new data, employing a small number (e.g., 5) of labeled outliers as a validation set for hyper-parameter tuning (unlabeled data are treated as negative samples). For the semi-supervised models, we vary the number of labeled outliers used in training, ranging from 5 to 30, while the remainder is used for testing. For GDOF, we consistently use 200 candidate inliers for negative sampling. Each experiment is independently repeated 10 times with randomly selected training sets, and the results are averaged for reporting.

\begin{table*}[!ht]
    \centering
    \tabcolsep=2pt
    \caption{AUC scores of detectors @ 5 labeled outliers. The best score is bolded, the 2nd rank is underlined.}\label{AUC}
    % \resizebox{\textwidth}{!}{
    \begin{tabular}{c|ccccc|cccccc}\toprule
    Datasets    & DeepSVDD & ECOD           & LUNAR          & DIF         & WFRDA          & DeepSAD     & REPEN       & DevNet         & FEAWAD      & PReNet         & GDOF           \\ \midrule
    Audiology   & 0.552    & {\ul 0.837}    & 0.622          & 0.512       & 0.834          & 0.592       & 0.633       & 0.619          & 0.616       & 0.639          & \textbf{0.871} \\
    Breast      & 0.581    & \textbf{0.659} & \textbf{0.659} & 0.609       & 0.657          & 0.570       & 0.649       & 0.507          & 0.528       & 0.493          & 0.580          \\
    Mushroom1   & 0.542    & {\ul 0.949}    & 0.814          & 0.767       & \textbf{0.971} & 0.942       & 0.929       & 0.902          & 0.910       & 0.886          & 0.924          \\
    Mushroom2   & 0.628    & 0.866          & {\ul 0.924}    & 0.726       & 0.882          & 0.904       & 0.888       & 0.892          & 0.828       & 0.890          & \textbf{0.972} \\
    Annealing   & 0.519    & 0.795          & 0.720          & 0.791       & 0.729          & 0.584       & 0.751       & {\ul 0.851}    & 0.832       & 0.847          & \textbf{0.877} \\
    Arrhythmia  & 0.619    & {\ul 0.807}    & 0.798          & 0.800       & 0.755          & 0.612       & 0.759       & 0.596          & 0.634       & 0.626          & \textbf{0.820} \\
    CreditA     & 0.865    & \textbf{0.991} & 0.837          & 0.867       & 0.975          & 0.739       & 0.948       & 0.805          & 0.709       & 0.813          & {\ul 0.977}    \\
    Sick        & 0.535    & 0.844          & 0.800          & 0.827       & 0.837          & {\ul 0.859} & 0.748       & 0.837          & {\ul 0.859} & 0.809          & \textbf{0.918} \\
    Thyroid     & 0.534    & 0.579          & 0.632          & 0.665       & 0.516          & 0.663       & 0.648       & 0.730          & 0.723       & {\ul 0.735}    & \textbf{0.748} \\
    Annthyroid  & 0.753    & 0.789          & 0.723          & 0.673       & 0.637          & 0.760       & 0.693       & {\ul 0.935}    & 0.794       & 0.931          & \textbf{0.981} \\
    Breastw     & 0.924    & {\ul 0.991}    & 0.974          & 0.748       & \textbf{0.992} & 0.838       & 0.987       & 0.764          & 0.827       & 0.704          & {\ul 0.991}    \\
    Cardio      & 0.522    & \textbf{0.935} & 0.644          & {\ul 0.934} & 0.914          & 0.733       & 0.889       & 0.775          & 0.795       & 0.776          & 0.924          \\
    Ionosphere  & 0.770    & 0.728          & \textbf{0.915} & {\ul 0.900} & 0.784          & 0.818       & 0.846       & 0.637          & 0.617       & 0.597          & 0.853          \\
    Mammography & 0.615    & \textbf{0.906} & 0.837          & 0.762       & 0.839          & 0.870       & 0.872       & 0.896          & 0.857       & {\ul 0.903}    & 0.901          \\
    Musk        & 0.669    & 0.959          & 0.877          & 0.996       & {\ul 0.999}    & 0.989       & {\ul 0.999} & 0.859          & 0.929       & 0.973          & \textbf{1.000} \\
    Optdigits   & 0.667    & 0.606          & 0.452          & 0.597       & 0.942          & 0.869       & 0.609       & 0.988          & {\ul 0.994} & \textbf{0.999} & 0.986          \\
    PageBlocks  & 0.708    & {\ul 0.914}    & 0.748          & 0.879       & 0.868          & 0.888       & 0.902       & 0.800          & 0.730       & 0.776          & \textbf{0.925} \\
    Waveform    & 0.628    & 0.601          & 0.744          & 0.740       & 0.699          & 0.733       & 0.668       & \textbf{0.859} & 0.821       & {\ul 0.854}    & 0.688          \\
    Wilt        & 0.460    & 0.394          & 0.509          & 0.358       & 0.331          & 0.696       & 0.341       & {\ul 0.909}    & 0.867       & \textbf{0.928} & 0.579          \\
    Yeast       & 0.491    & 0.445          & 0.431          & 0.394       & 0.395          & 0.474       & 0.383       & \textbf{0.602} & 0.573       & {\ul 0.600}    & 0.464          \\ \midrule
    Average     & 0.629    & 0.780          & 0.733          & 0.727       & 0.778          & 0.757       & 0.757       & 0.788          & 0.772       & {\ul 0.789}    & \textbf{0.849} \\ \bottomrule
    \end{tabular}
    % }
    \end{table*}

\subsection{Comparison methods and settings}
In this experiment, we evaluate GDAD's performance against 10 baseline methods, including 5 unsupervised and 5 semi-supervised algorithms. Most of these are implemented in the library PyOD \cite{zhao2019PyOD} and WSAD \cite{Jiang2023WSAD}, except for WFRDA.

    \textbf{Unsupervised comparison methods:}
    \begin{itemize}
    \item DeepSVDD (2018) \cite{Ruff2018DeepSVDD}: A distance-based model that projects objects into a minimal-volume hypersphere. The epoch number is optimized among \{20, 50, 100, 200\}.
    \item ECOD (2022) \cite{li2022ecod}: A statistical model that estimates the cumulative distribution of objects without parameters.
    \item LUNAR (2022) \cite{Goodge2022LUNAR}: A distance-based method that employs graph neural networks to predict outliers. The nearest neighbors is selected from \{5, 10, 20, 50\}.
    \item DIF (2023) \cite{Xu2023DIF}: An extension of IForest that learns random representation ensembles with neural networks. The estimators is chosen from \{5, 10, 50, 100\}.
    \item WFRDA (2023) \cite{Yuan2023WFRDA}: A density-based method that introduces the fuzzy-rough density to describe outliers. The fuzzy radius is selected from 0.1 to 2.0 {in 0.1 increments}. 
    \end{itemize}

    \textbf{Semi-supervised detection methods:}
    \begin{itemize}
    \item REPEN (2018) \cite{Pang2018RAMODO}: A rank-based method that learns the representations of high-dimensional data. 
    \item DevNet (2019) \cite{Pang2019DevNet}: A rank-based method that learns feature representation {using deviation loss}. %The confidence margin $a$ is set to 5.
    \item DeepSAD (2020) \cite{Ruff2020DeepSAD}: A deep learning-based model that optimizes the inverse distances of anomalies. %The balancing parameter $\eta$ is set to 1.
    \item FEAWAD (2022) \cite{Zhou2022FEAWAD}: A rank-based model that utilizes a deep autoencoder to fit the normal data.
    \item PReNet (2023) \cite{Pang2023PReNet}: A rank-based method that employs ordinal regression without involving any assumptions about outlier scores.
    \end{itemize}

\setcounter{footnote}{0}
\subsection{Evaluation metrics}
    Our evaluation of detection methods is based on two metrics: AUC (Area Under the receiver operating characteristic Curve) and AP (Average Precision), {both implemented in Python's Scikit-learn library. AUC, ranging} from 0 to 1, measures the algorithm's class discrimination capability, with higher values indicating superior performance. AP, focusing on the rate of correctly identified positives, is the weighted mean of precision at various thresholds. While AUC is robust against class imbalance, AP is more indicative in imbalanced settings.
    
\subsection{Experimental results and analysis}
\subsubsection{Overall performance}
This part focuses on the performance of detectors in real-world datasets with limited labeled outliers. We train detectors with only 5 labeled outliers, using the rest of the unlabeled data as the test set. For fairness, the same training set is used for hyperparameter tuning in unsupervised models.

\begin{table*}[!ht]
    \centering
    \tabcolsep=2pt
    \caption{AP scores of detectors @ 5 labeled outliers. The best score is bolded, the 2nd rank is underlined.}\label{PR}
    % \resizebox{\textwidth}{!}{
    \begin{tabular}{c|ccccc|cccccc} \toprule
    Datasets    & DeepSVDD & ECOD           & LUNAR          & DIF         & WFRDA          & DeepSAD     & REPEN       & DevNet         & FEAWAD      & PReNet         & GDOF           \\ \midrule
    Audiology   & 0.292 & 0.649          & 0.436          & 0.260       & {\ul 0.706}    & 0.341          & 0.400 & 0.542          & 0.473          & 0.515          & \textbf{0.800} \\
    Breast      & 0.378 & \textbf{0.467} & 0.438          & 0.379       & {\ul 0.462}    & 0.355          & 0.418 & 0.320          & 0.338          & 0.318          & 0.417          \\
    Mushroom1   & 0.179 & 0.483          & 0.236          & 0.098       & \textbf{0.899} & 0.675          & 0.411 & 0.835          & 0.852          & 0.830          & {\ul 0.894}    \\
    Mushroom2   & 0.398 & 0.365          & 0.583          & 0.189       & 0.673          & 0.610          & 0.401 & {\ul 0.835}    & 0.778          & 0.828          & \textbf{0.915} \\
    Annealing   & 0.091 & 0.199          & 0.106          & 0.104       & 0.110          & 0.085          & 0.115 & 0.494          & 0.482          & {\ul 0.506}    & \textbf{0.524} \\
    Arrhythmia  & 0.233 & 0.448          & 0.428          & {\ul 0.458} & 0.373          & 0.227          & 0.373 & 0.251          & 0.301          & 0.279          & \textbf{0.518} \\
    CreditA     & 0.479 & \textbf{0.916} & 0.388          & 0.298       & 0.855          & 0.331          & 0.643 & 0.621          & 0.501          & 0.650          & {\ul 0.898}    \\
    Sick        & 0.043 & 0.063          & 0.056          & 0.057       & 0.059          & 0.107          & 0.051 & 0.315          & {\ul 0.355}    & 0.302          & \textbf{0.363} \\
    Thyroid     & 0.011 & 0.009          & 0.013          & 0.012       & 0.008          & 0.018          & 0.011 & 0.051          & 0.046          & \textbf{0.067} & {\ul 0.059}    \\
    Annthyroid  & 0.209 & 0.269          & 0.185          & 0.227       & 0.169          & 0.253          & 0.198 & {\ul 0.616}    & 0.370          & 0.592          & \textbf{0.781} \\
    Breastw     & 0.805 & \textbf{0.984} & 0.924          & 0.465       & {\ul 0.979}    & 0.777          & 0.967 & 0.791          & 0.839          & 0.735          & 0.977          \\
    Cardio      & 0.160 & 0.562          & 0.222          & 0.605       & 0.516          & 0.310          & 0.474 & 0.543          & {\ul 0.621}    & 0.577          & \textbf{0.659} \\
    Ionosphere  & 0.589 & 0.638          & \textbf{0.904} & {\ul 0.868} & 0.655          & 0.764          & 0.777 & 0.611          & 0.585          & 0.599          & 0.789          \\
    Mammography & 0.056 & 0.432          & 0.123          & 0.128       & 0.094          & 0.314          & 0.177 & {\ul 0.507}    & 0.275          & \textbf{0.515} & 0.446          \\
    Musk        & 0.307 & 0.504          & 0.646          & 0.896       & {\ul 0.981}    & 0.893          & 0.972 & 0.802          & 0.902          & 0.945          & \textbf{1.000} \\
    Optdigits   & 0.053 & 0.033          & 0.028          & 0.033       & 0.366          & 0.206          & 0.036 & {\ul 0.968}    & 0.842          & \textbf{0.985} & 0.812          \\
    PageBlocks  & 0.339 & 0.517          & 0.354          & 0.506       & 0.373          & 0.555          & 0.538 & \textbf{0.580} & 0.522          & 0.568          & {\ul 0.575}    \\
    Waveform    & 0.070 & 0.038          & 0.138          & 0.063       & 0.047          & \textbf{0.235} & 0.053 & 0.191          & 0.185          & {\ul 0.210}    & 0.055          \\
    Wilt        & 0.045 & 0.041          & 0.049          & 0.037       & 0.036          & 0.106          & 0.036 & 0.339          & \textbf{0.442} & {\ul 0.393}    & 0.061          \\
    Yeast       & 0.328 & 0.331          & 0.314          & 0.291       & 0.324          & 0.319          & 0.287 & \textbf{0.431} & 0.409          & {\ul 0.427}    & 0.332          \\ \midrule
    Average     & 0.253 & 0.397          & 0.329          & 0.299       & 0.434          & 0.374          & 0.367 & 0.532          & 0.506          & {\ul 0.542}    & \textbf{0.594} \\ \bottomrule
    \end{tabular}
    % }
    \end{table*}

\begin{table*}[!ht]
    \centering
    \tabcolsep=2pt
    \caption{Average scores {for different} data types.}\label{tab_DType}
    % \resizebox{\textwidth}{!}{
    \begin{tabular}{cc|ccccc|cccccc}
    \toprule
    Metric                   & Data type & DeepSVDD& ECOD& LUNAR& DIF& WFRDA& DeepSAD& REPEN& DevNet& FEAWAD& PReNet& GDOF\\         \midrule
    \multirow{3}{*}{AUC}     & Categorical & 0.576 & 0.828 & 0.755 & 0.654 & 0.836 & 0.752 & 0.775 & 0.730 & 0.721 & 0.727 & \textbf{0.837} \\
                             & Mixed       & 0.614 & 0.803 & 0.757 & 0.790 & 0.762 & 0.691 & 0.771 & 0.764 & 0.751 & 0.766 & \textbf{0.868} \\
                             & Numerical   & 0.655 & 0.752 & 0.714 & 0.726 & 0.764 & 0.788 & 0.744 & 0.820 & 0.800 & 0.822 & \textbf{0.845} \\ \midrule
    \multirow{3}{*}{AP}      & Categorical & 0.312 & 0.491 & 0.423 & 0.232 & 0.685 & 0.495 & 0.408 & 0.633 & 0.610 & 0.623 & \textbf{0.757} \\
                             & Mixed       & 0.171 & 0.327 & 0.198 & 0.186 & 0.281 & 0.154 & 0.239 & 0.346 & 0.337 & 0.361 & \textbf{0.472} \\
                             & Numerical   & 0.269 & 0.395 & 0.353 & 0.374 & 0.413 & 0.430 & 0.410 & 0.580 & 0.545 & \textbf{0.595} & 0.590 \\ \bottomrule
    \end{tabular}
    % }
    \end{table*}
   
    The experimental comparison results, detailed in Tables ~\ref{AUC} and ~\ref{PR} with 5 labeled outliers, show the effectiveness of GDOF across various datasets. GDOF achieves the highest AUC in 9 out of 20 datasets, including notable performances in Audiology (0.871), Mushroom2 (0.972), and Annthyroid (0.981). It also ranks second in AUC on CreditA (0.977) and Breastw (0.991) datasets. For AP, GDOF excels in 8 out of 20 datasets and holds comparable positions in 4 others, including Mushroom1 and Thyroid. Overall, GDOF secures the highest average scores across all datasets, demonstrating its robustness in outlier detection with limited labeled data in diverse applications.

    However, GDOF's performance is not as strong in some cases when compared to other semi-supervised methods. Notably, it falls behind on datasets Waveform, Wilt, and Yeast. For example, DevNet surpasses GDOF in AUC by 24.85\% on Waveform and 29.74\% on Yeast. Similarly, PReNet's AUC on Wilt is 60.28\% higher than GDOF's. In terms of AP, DeepSAD, FEAWAD, and DevNet outperform GDOF on Waveform, Wilt, and Yeast, respectively. This limitation may stem from GDOF's assumption of dimension independence, which might not be well-suited for datasets with complex inter-attribute dependencies.

\begin{figure*}[!ht]
    \centering
    \makebox[\textwidth][c]{\includegraphics[width=0.95\textwidth]{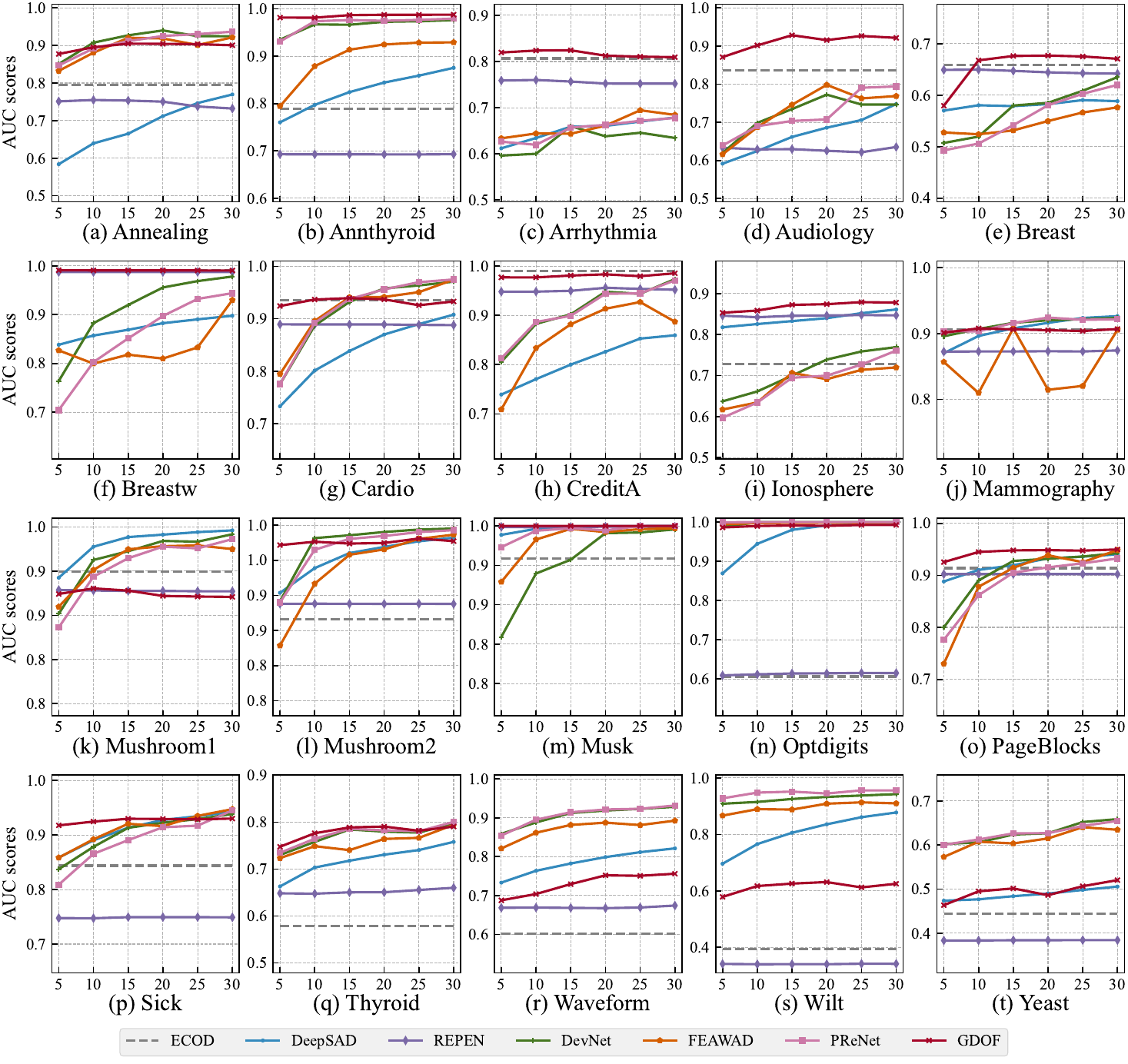}}
    \caption{AUC scores across {different numbers of labeled outliers. The horizontal axis indicates the number of labeled outliers, ranging from 5 to 30.}}
    \label{fig_train}
    \end{figure*}
    
\subsubsection{Detection performances in handling heterogeneous data} 
    This part examines how models perform in heterogeneous datasets characterized by multiple attribute values.
    Table \ref{tab_DType} presents a summary of the average performance of comparison detectors across three types of data: categorical, mixed, and numerical. In terms of both AUC and AP, GDOF stands out, achieving the highest average scores across 4 categorical and 5 mixed datasets. Additionally, GDOF's performance on numerical datasets is on par with that of its competitors. Specifically, GDOF surpasses PReNet by margins of 15.1\% and 13.3\% in AUC for categorical and mixed datasets, respectively. However, in numerical datasets, GDOF's AP score of 0.59 is slightly lower compared to PReNet's 0.595.
    The effectiveness of GDOF in this context can be attributed to its use of Fuzzy granules, which allows it to represent heterogeneous attributes without the need for transforming data types. This approach helps in preserving more critical information that is beneficial for detecting outliers.

\begin{figure*}[!ht]
    \centering
    \makebox[\textwidth][c]{\includegraphics[width=.95\textwidth]{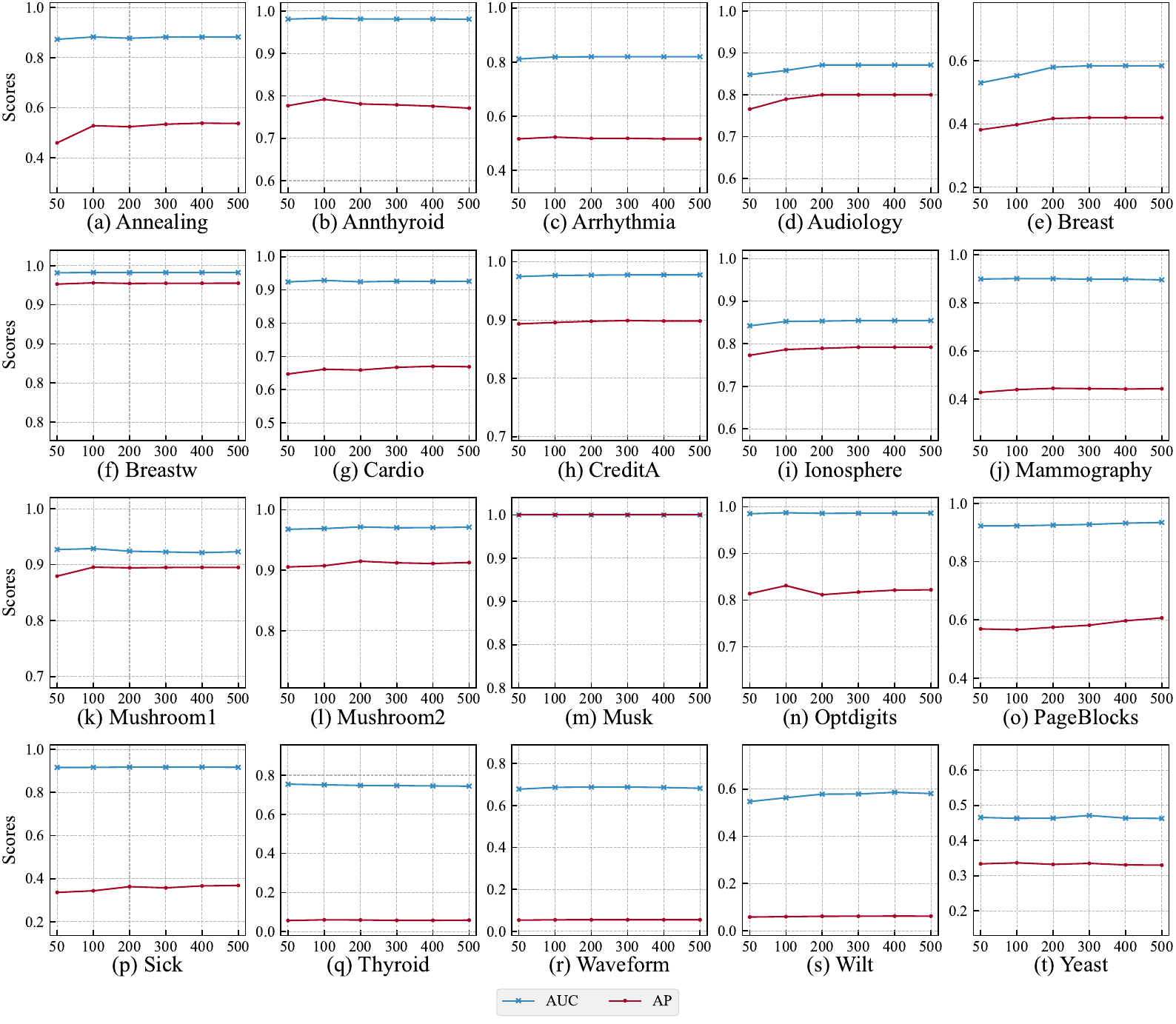}}
    \caption{GDOF's performances across various numbers of normal {objects. The horizontal axis indicates the number of normal objects, ranging from 50 to 500.}}
    \label{fig_N_neg}
    \end{figure*}
   
\subsubsection{Model performance across various numbers of labeled outliers}
    This part explores detection performance across different {numbers of labeled outliers} and investigates the improvements of the semi-supervised methods over the unsupervised algorithm under these conditions. The number of labeled outliers for training varies from 5 to 30, {and the unlabeled data are treated as normal points} during training. This setting aligns with prior studies \cite{Pang2018RAMODO, Pang2019DevNet, Pang2023PReNet} and is akin to training with noisy labels. However, we do not artificially adjust outlier contamination levels since it is more practical in real-world {applications with various numbers of outliers.}
       
    Figure \ref{fig_train} {shows the AUC scores} across multiple numbers of labeled outliers.
    These semi-supervised algorithms tend to improve with the increase of the number of labeled outliers, as more labeled data provide more information for training. However, {some methods fall} with more labeled data in some datasets, e.g., FEAWAD and DevNet on Audiology, GDOF on Wilt. {This may be due to the newly added data exhibiting anomalous behavior that conflicts with existing data, thereby reducing detection performance.}
    Moreover, we also noticed that GDOF does not improve significantly with an increasing {number of labeled outliers in some datasets}, e.g., GDOF's improvements on Waveform, Wilt, and Yeast are much lower than that of DevNet, PReNet or FEAWAD.
    {This is because GDOF relies on} only a few unlabeled data to learn the outlier behavior, making it difficult to capture the full data distribution in some cases, which limits detection performance.
        
    Furthermore, most semi-supervised detection algorithms tend to outperform the leading unsupervised method, ECOD, when only a minimal number of labeled outliers are available. For example, on 14 of 20 datasets, GDOF demonstrates superior performance over ECOD using merely 5 labeled outliers. Similarly, DeepSVD, DevNet, FEAWAD, and PReNet surpass ECOD on over 13 datasets when given 30 labeled outliers in terms of AUC. Among these semi-supervised methods, GDOF emerges as the most data-efficient.
    % The effectiveness of GDOF can be attributed to its unique approach of utilizing granule density for outlier scoring. This design enables GDOF to seamlessly incorporate information from a large amount of unlabeled data and a small set of labeled outliers, thereby enhancing GDOF's capability to identify potential outliers within the data.

\subsubsection{Detection performances with various parameter settings}
    As GDOF involves sampling candidate normal objects, the number of negative instances $N_-$ could influence its detection efficiency. This part delves into how varying $N_-$ impacts GDOF's performance. We experiment with different settings for $N_-$, i.e., 50, 100, 200, 300, 400, 500, while maintaining a constant count of 5 labeled outliers.
   
    The results, illustrated in Figure \ref{fig_N_neg}, reveal notable trends in GDOF's performance. Both the AUC and AP scores of GDOF generally show a slightly upward trend as $N_-$ increases in most datasets. Also, these scores tend to stabilize when $N_-$ reaches a certain level, such as 100 or 200. This observation underscores that GDOF's effectiveness is not heavily dependent on a large number of normal objects. It exhibits stable performance across these parameter variations, highlighting its robustness in the number of selected negative instances and its efficiency even with a limited number of normal objects.
    Furthermore, these results demonstrate the efficiency of GDOF's negative sampling approach, which utilizes a small number of normal points to construct outlier factors.

\section{Conclusion}
In this paper, we presented the Granule Density-based Outlier Factor (GDOF), a label-informed outlier detection method for heterogeneous data grounded in Granular Computing and Fuzzy Sets. GDOF effectively represents diverse data types through a label-informed fuzzy granulation process and accurately estimates data density with the design of granule density. Finally, we integrate granule densities from individual attributes for outlier scoring by assessing attribute relevance. 
Extensive experiments across 20 diverse real-world datasets indicate that GDOF excels in handling heterogeneous data with a minimal number of labeled outliers for outlier detection. Notably, GDOF outperforms competing detectors and maintains stable performance across different parameter settings. However, challenges arise in datasets with complex inter-attribute dependencies, and the algorithm shows limited improvement with increased levels of label supervision. 
Our future work will focus on enhancing the detection of attribute interdependencies and leveraging the implicit knowledge within the data to further improve model performance.

\bibliographystyle{ieeetr}
\bibliography{mybib}

\end{document}